\documentclass[pdflatex,sn-mathphys-num]{sn-jnl}

\makeatletter
\let\orig@normalsize=\normalsize
\def\normalsize{
    \orig@normalsize
    \setlength\abovedisplayskip{11pt}
    \setlength\belowdisplayskip{11pt}
    \setlength\abovedisplayshortskip{11pt}
    \setlength\belowdisplayshortskip{11pt}
}

\usepackage{graphicx}
\usepackage{multirow}
\usepackage{amsmath,amssymb,amsfonts}
\usepackage{mathtools}
\usepackage{amsthm}
\usepackage{aliascnt}
\usepackage{mathrsfs}
\usepackage{anyfontsize}
\usepackage{etoolbox}
\usepackage{xcolor}
\usepackage{textcomp}
\usepackage{manyfoot}
\usepackage{bm}
\usepackage{booktabs}
\usepackage{algorithm}
\usepackage{algorithmicx}
\usepackage{algpseudocode}
\usepackage{listings}

\let\mathscr\mathscr
\newtheorem{theorem}{Theorem}

\newaliascnt{proposition}{theorem}
\newtheorem{proposition}[proposition]{Proposition}
\aliascntresetthe{proposition}

\newaliascnt{lemma}{theorem}
\newtheorem{lemma}[lemma]{Lemma}
\aliascntresetthe{lemma}

\newaliascnt{corollary}{theorem}

\aliascntresetthe{corollary}

\newtheorem{definition}{Definition}
\newtheorem{remark}{Remark}

\DeclareMathOperator{\argmin}{argmin}
\DeclareMathOperator{\rank}{rank}
\DeclareMathOperator{\kNN}{kNN}
\DeclareMathOperator{\Err}{Err}
\DeclareMathOperator{\Var}{Var}

\preto\equation{\vspace*{-1em}}
\preto\align{\vspace*{-1em}}
\preto\gather{\vspace*{-1em}}

\pretocmd{\thebibliography}{%
  \interlinepenalty=10000\relax
  \widowpenalty=10000\relax
  \clubpenalty=10000\relax
}{}{}

\begin{document}

\title[Sheaf-Laplacian Obstruction and Projection Hardness for Cross-Modal Compatibility on a Modality-Independent Site]{Sheaf-Laplacian Obstruction and Projection Hardness for Cross-Modal Compatibility on a Modality-Independent Site}


\author[1,2]{\fnm{Tibor} \sur{Sloboda} \orcid{https://orcid.org/0000-0001-6817-6297} \email{tibor.sloboda@stuba.sk}}

\affil[1]{\orgdiv{UPAI, Faculty of Informatics and Information Technologies},
  \orgname{Slovak University of Technology in Bratislava},
  \orgaddress{\city{Bratislava}, \country{Slovakia}}}

\affil[2]{\orgname{aleph0 s.\ r.\ o.},
  \orgaddress{\city{Bratislava}, \country{Slovakia}}}


\abstract{Cross-modal representations vary in how easily they can be aligned, and compatibility is generally non-transitive: two modalities may align through an intermediate modality at lower complexity than through a direct map. We introduce a reference formalism that evaluates all modalities on a fixed neighborhood site and defines two directed invariants. Projection hardness \(H_{a\to b}(\varepsilon)\) is the minimum complexity within a nested Lipschitz-controlled family required to reach error \(\varepsilon\). For a declared local projection family, sheaf-Laplacian obstruction \(C_{a\to b}(\varepsilon)\) is the minimum variation of locally fitted projection parameters required to reach the same error. Under identity restrictions, obstruction is the graph Dirichlet energy of a vector-valued parameter field; the sheaf formulation identifies zero energy with successful gluing and extends to edge-dependent transports and heterogeneous parameter spaces. We relate obstruction to the site spectral gap and excess global-map error, and construct ReLU examples showing non-transitive compatibility and a quadratic separation between staged and direct width. Controlled synthetic calibrations recover the predicted hardness separation, cut-induced obstruction scaling, and sensitivity to the fixed site graph.}




\maketitle

\section{Introduction}

Modern multimodal systems routinely produce embeddings for heterogeneous data types (text, audio, images, video, proprioception, graphs) and align them in a joint representational space. Empirically, some modality pairs admit simple, approximately linear alignments. Others require nonlinear models, extensive supervision, or remain difficult despite large model capacity.

Alignment is also not reliably transitive: modality $a$ may align well to $c$ and $c$ to $b$, with poor direct alignment from $a$ to $b$. A two-stage alignment $a\to c\to b$ may require substantially less complexity than the direct $a\to b$. These observations motivate a theory that is comparable across modality pairs, local-to-global in structure, and capable of expressing non-transitivity and bridging as formal statements.

This paper proposes a reference formalism that fixes the choices which otherwise make results difficult to compare. The central requirement is that all modalities are analyzed on the same base domain: a common neighborhood site on sample indices. Sheaf operators, Laplacians, energies, and spectra are therefore defined on the same graph.

\noindent
\\On that site we define two invariants for each directed pair $(a\to b)$:

\begin{itemize}
\item \textbf{Projection hardness} $H_{a\to b}(\varepsilon)$: the lowest complexity, within a nested Lipschitz-controlled projection family, needed for a global map to achieve alignment error at most $\varepsilon$.
\item \textbf{Sheaf-Laplacian obstruction} $C_{a\to b}(\varepsilon)$: the minimal spatial variation in a locally fit field of projection parameters, within a declared local projection family, required to reach the same target error over the fixed site.
\end{itemize}

\noindent
Hardness measures global expressivity across the nested family. Obstruction measures local-to-global inconsistency within the declared local family: good local alignments may require substantial parameter variation to glue across neighborhoods.

We compute the obstruction through a projection-parameter sheaf on the base graph. Under identity restrictions, its 0-Laplacian energy is exactly the quadratic smoothness penalty for a vector-valued graph signal. The sheaf formulation identifies zero energy with a global section and extends the same gluing criterion to heterogeneous local parameter spaces and edge-dependent transports.

\subsection*{Contributions}
In this paper, we make the following contributions.

\begin{enumerate}
\item We fix a modality-independent site on which all modalities are evaluated, together with finite-dimensional real inner-product spaces and linear maps as the stalk category. This is necessary because energies and spectra computed on different neighborhood graphs cannot be compared directly.
\item We introduce projection hardness using normalized nested projection families indexed by rank, width, or another fixed complexity parameter. The common family and normalization provide the same complexity scale for every directed modality pair.
\item We construct a projection-parameter sheaf whose 0-Laplacian energy measures the variation required by locally fitted projection maps. This separates the complexity of finding one global map from the consistency of local maps across the site.
\item We derive relations connecting the sheaf spectral gap with alignment stability and bound the excess error of a global map in terms of obstruction energy. These results give the obstruction quantity a direct interpretation beyond its use as a regularization value.
\item We give explicit constructions of non-transitivity and bridging. In the one-dimensional ReLU construction, each stage of a two-stage alignment has width $O(w)$, whereas exact direct alignment at the same depth requires width $\Omega(w^2)$.
\item We report controlled synthetic calibrations for projection-hardness separation, obstruction scaling, and site sensitivity. Their analytically known reference values make it possible to verify the computations directly, and the accompanying protocol specifies how the same quantities should be reported for paired multimodal data.
\end{enumerate}

\section{Related Work}

\subsection*{Cross-modal representation learning} Aligning representations from different modalities is a central problem in multimodal machine learning~\cite{baltruvsaitis2018multimodal}. Standard approaches range from canonical correlation analysis (CCA)~\cite{hotelling1936relations} and its kernelized variants~\cite{lai2000kernel} to contrastive methods such as CLIP~\cite{radford2021learning}, which align image-text pairs through large-scale pre-training. These methods achieve strong empirical results. Their global alignment objectives generally do not model local geometric structure explicitly or quantify alignment hardness separately from global error.

\subsection*{Sheaf theory in deep learning}
Cellular sheaves provide a framework for modeling relational data with heterogeneous stalks and consistency constraints~\cite{curry2014sheaves,ghrist2014elementary}. Recent work has applied sheaf Laplacians to graph neural networks, yielding Sheaf Neural Networks (SNNs) and sheaf diffusion models that generalize standard message passing by respecting local topological structure~\cite{hansen2020generalized,bodnar2022neural,calmon2023sheaf}. We use sheaves here as an analytical tool to measure the obstruction to gluing local alignment maps. The projection-parameter sheaf encodes spatial variation in model parameters, connecting sheaf-theoretic energies to regularized regression objectives.

\subsection*{Hardness and complexity of representation alignment}
The difficulty of learning alignments between representation spaces has been studied from several angles, including sample complexity bounds for domain adaptation~\cite{ben2010theory,mansour2009domain}, minimax rates for transfer learning~\cite{cai2021transfer}, and the dimensionality of optimal transport mappings~\cite{perrot2016mapping}. Low-capacity alignment maps also underlie model stitching, where compatibility is assessed through a restricted trainable interface~\cite{bansal2021stitching}, while recent representational-alignment work organizes transformation choices along a spectrum from identity and linear maps to constrained nonlinear functions~\cite{sucholutsky2023aligned}. Projection hardness turns this spectrum into a tolerance-indexed empirical approximation-complexity quantity: for a normalized nested projection family, it identifies the first class that reaches the target alignment error on the fixed sample set. Statistical complexity measures such as VC dimension and Rademacher complexity then bound generalization within the selected class~\cite{bartlett2002rademacher,mohri2018foundations}. The two quantities therefore describe the approximation and generalization requirements of the same projection family.

\subsection*{Non-transitivity and composition in multimodal systems}
Empirical studies show that modality alignment can fail to be transitive: modality $A$ may align to $B$ and $B$ to $C$, with poor direct alignment from $A$ to $C$~\cite{liang2022foundations,fei2023bridging}. This phenomenon is relevant for bridge modalities or pivot languages in translation~\cite{firat2016zero,gu2018universal}. Our framework formalizes it through bridging via composed projection families, giving explicit hardness bounds that demonstrate when mediation reduces complexity. The ReLU construction in \autoref{sec:relu_construction} adapts the triangle-wave composition used in neural-network depth-separation arguments~\cite{telgarsky2016benefits} to obtain a staged-versus-direct width separation for alignment.

\begin{figure*}[!ht]
  \centering
  \includegraphics[width=0.6\textwidth]{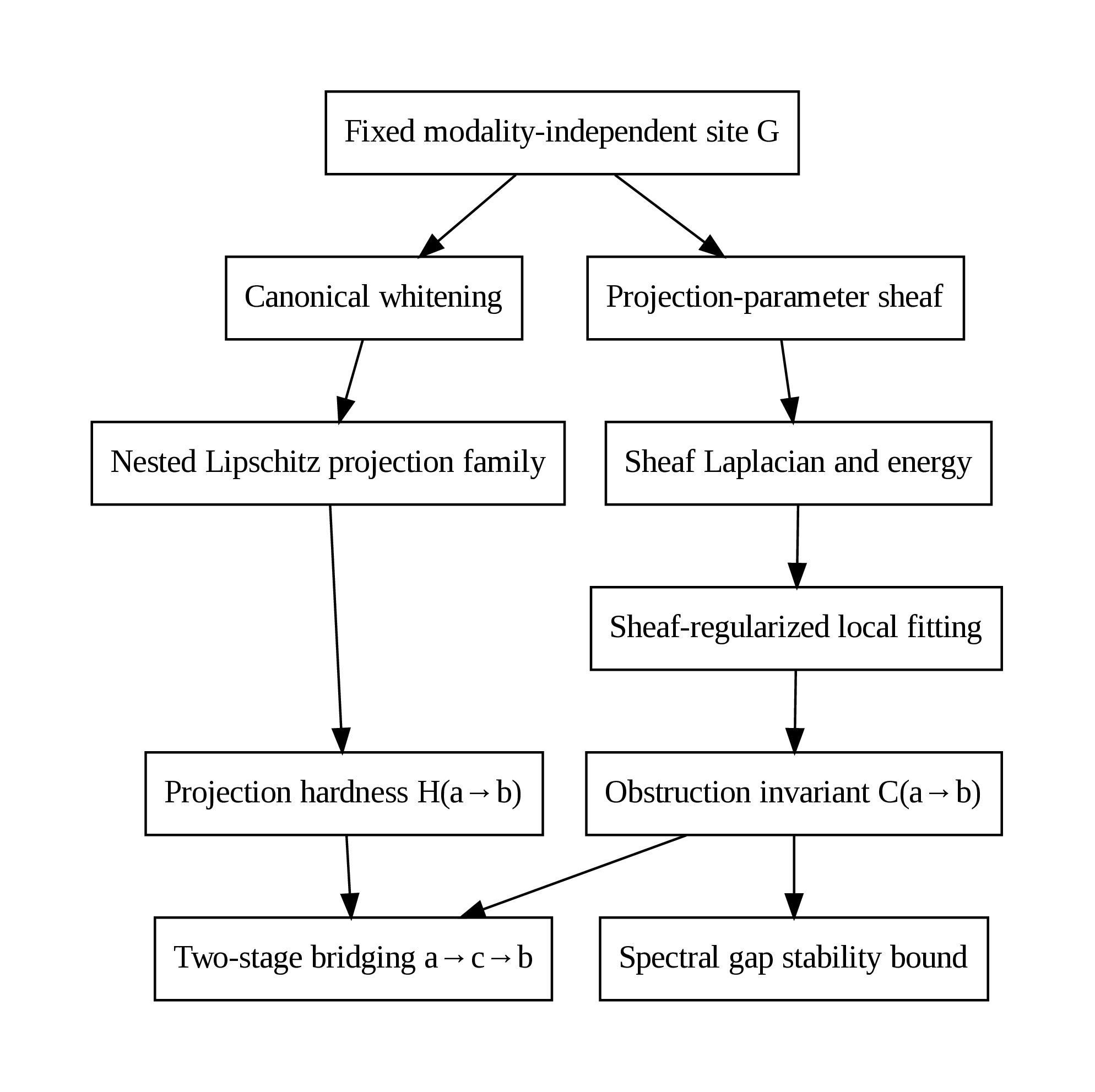}
  \caption{Conditional dependency graph of the framework. Nodes represent definitions, constructions, or results. A directed edge indicates a logical prerequisite.}
  \label{fig:dependency_graph}
\end{figure*}

\section{Framework and Setup}

\subsection*{Data, modalities, and embeddings}
Let $V=\{1,\dots,n\}$ index samples. Let $\mathcal{M}$ denote a finite set of modalities. For each modality $m\in\mathcal{M}$, an encoder (trained independently of other modalities) produces an embedding

\nopagebreak
\begin{equation}
z_i^{(m)} \in \mathbb{R}^{d_m} \qquad i\in V
\end{equation}

We compare modalities via maps between their embedding spaces. We normalize embeddings and, when needed, fix a common output dimension $d$ (either by design or through a fixed projection convention). The formalism is presented in the common-dimension case; for $d_a\neq d_b$, the source and target spaces are $\mathbb{R}^{d_a}$ and $\mathbb{R}^{d_b}$, respectively.

\subsection*{A shared base site as a comparability constraint}
A recurring source of ambiguity in cross-modal analysis is that different modalities induce different neighborhood graphs when neighborhoods are built in embedding space. Sheaf Laplacians constructed on these different graphs cannot be compared directly because their energies and spectra refer to different sites.

\noindent
\\We enforce the following requirement throughout:

\begin{definition}[Modality-independent site requirement]
All modalities are analyzed on the same base domain: a fixed simplicial complex (in practice, a fixed graph) on sample indices $V$.
\end{definition}

All sheaf Laplacians, energies, and obstruction measures are computed on this shared site.

\subsection{Sheaves and the sheaf Laplacian}

\subsection*{Stalk category}

\begin{definition}[Stalk category]
Let $\mathbf{Hilb}_{\mathrm{fd}}$ denote the category whose objects are finite-dimensional real inner-product spaces and whose morphisms are linear maps.
\end{definition}

Inner products supply adjoints, quadratic energies, Laplacians, and spectral gaps. Linear restriction maps interact cleanly with regression-style estimators.

\subsection*{Base domain as a simplicial complex}

Let $K$ be a simplicial complex on the vertex set $V$, intended to approximate locality in an underlying semantic domain. In most computations we work on the $1$-skeleton graph $G=(V,E)$ of $K$ (possibly with edge weights).

\begin{remark}
For comparability, $G$ is fixed once and used across all modalities and all modality pairs. Its construction is part of the methodology and is set before the alignment procedure.
\end{remark}

\begin{samepage}
\subsection*{Cellular sheaf on a graph}

We work with a cellular sheaf on a graph.

\begin{definition}[Cellular sheaf on a graph]
Let $G=(V,E)$ be an undirected graph. A cellular sheaf $\mathcal{F}$ of $\mathbf{Hilb}_{\mathrm{fd}}$-objects on $G$ consists of:

\begin{itemize}
\item a stalk $\mathcal{F}(v)$ for each vertex $v\in V$,
\item a stalk $\mathcal{F}(e)$ for each edge $e\in E$,
\item for each incidence $v\leq e$ (i.e., $v$ is an endpoint of $e$), a restriction map

\nopagebreak
\begin{equation}
\rho_{e\leftarrow v}:\mathcal{F}(v)\to \mathcal{F}(e)
\end{equation}

\end{itemize}

\end{definition}
\end{samepage}

A $0$-cochain is an assignment $c(v)\in \mathcal{F}(v)$ for each vertex. A $1$-cochain is an assignment $x(e)\in \mathcal{F}(e)$ for each edge. Fix an orientation for each edge $e=(u,v)$. The coboundary operator $\delta^0: C^0(G;\mathcal{F}) \to C^1(G;\mathcal{F})$ acts by

\nopagebreak
\begin{equation}
(\delta^0 c)(e) = \rho_{e\leftarrow u}(c(u)) - \rho_{e\leftarrow v}(c(v))
\end{equation}

\noindent
The inner products on stalks give $\delta^0$ an adjoint $(\delta^0)^\top$. The degree-$0$ sheaf Laplacian is

\nopagebreak
\begin{equation}
\Delta^0 = (\delta^0)^\top \delta^0
\end{equation}

\begin{definition}[Sheaf inconsistency energy]
For a $0$-cochain $c\in C^0(G;\mathcal{F})$, define the sheaf inconsistency energy

\nopagebreak
\begin{equation}
E_{\mathcal{F}}(c) := \|\delta^0 c\|^2 = \langle c, \Delta^0 c\rangle
\end{equation}

\end{definition}

\begin{remark}[Relation to Dirichlet energy on graphs]
The quadratic form $E_{\mathcal F}(c)=\|\delta^0 c\|^2=\langle c,\Delta^0 c\rangle$ is a Dirichlet-type energy associated with the sheaf Laplacian. In the special case where $\mathcal F$ is the constant sheaf with stalk $\mathbb R$ and identity restrictions, $\Delta^0$ reduces to the usual graph Laplacian and $E_{\mathcal F}(c)$ becomes the classical graph Dirichlet energy $\sum_{(u,v)\in E} (c(u)-c(v))^2$. For constant stalk $\mathbb R^p$ this is the Dirichlet energy of a vector-valued graph signal, i.e.\ the sum of the scalar Dirichlet energies of its coordinates~\cite{evans1998partial}.
\end{remark}

\begin{remark}[Interpretation]
$E_{\mathcal{F}}(c)=0$ if and only if $c$ is a global section, meaning that it glues across all edges. The spectrum of $\Delta^0$ quantifies stability: a large spectral gap implies that local consistency strongly constrains global consistency.
\end{remark}

\subsection{Canonical Neighborhood Construction}

To satisfy the modality-independent site requirement, we must build $G$ without privileging any modality embedding as the ground truth.

\subsubsection*{Synthetic setting: ground-truth latent site}
In synthetic settings where ground-truth latent states $s_i$ are available with a metric $d_X$, we define a $k$-nearest-neighbor graph $G_X=(V,E_X)$ by

\nopagebreak
\begin{equation}
(i,j)\in E_X \;\Longleftrightarrow\; j \in \kNN_{k_{\mathrm{nn}}}(i; d_X)
\end{equation}

\noindent
optionally with RBF weights

\nopagebreak
\begin{equation}
w_{ij}=\exp\!\left(-\frac{d_X(s_i,s_j)^2}{\sigma^2}\right)
\end{equation}

\noindent
This yields a canonical cover $U_i=\{i\}\cup N_{G_X}(i)$.

\subsubsection*{Practical constructions without latent states}
Without ground-truth latent states, the site should be built from information that is shared across modalities and fixed before evaluating alignment. Common choices include:

\begin{itemize}
\item label or metadata graphs, where edges connect samples sharing class labels, object identity, patient identity, speaker identity, scene identity, or other external annotations;
\item temporal or causal graphs for sequences, where edges encode time adjacency, event adjacency, or known causal links;
\item externally specified relational graphs, such as geospatial proximity, ontology edges, citation links, interaction graphs, or human-curated neighborhood structures;
\item consensus graphs formed from multiple modality-specific neighborhood graphs, using union, intersection, or mutual-nearest-neighbor rules after normalization.
\end{itemize}

Consensus construction requires additional care because the graph depends on modality-specific embeddings. When no external semantic graph is available, the consensus rule must be fixed before alignment is evaluated and reported with the results.

\subsubsection*{Site sensitivity as part of the protocol}
Because obstruction is graph-relative, practical use should report the graph construction together with basic graph statistics such as edge density, connected components, weighting rule, and $\lambda_2(G)$. Sensitivity to graph choices should be treated as part of the analysis: varying $k$, edge weights, label granularity, or union/intersection rules can reveal whether a high obstruction value is stable or an artifact of the chosen site. The controlled experiments in \autoref{fig:exp-site-sensitivity} use this principle directly by changing edge placement with the data and projection family fixed.

\subsection{Projection Families and a Comparable Notion of Hardness}

We now fix the map classes used to quantify global alignment complexity. Comparability across modality pairs requires consistent normalization, complexity indices, and Lipschitz control.

\subsubsection*{Canonical whitening}

For each modality $m\in\mathcal{M}$, compute the empirical mean $\mu_m$ and covariance $\Sigma_m$ of $\{z_i^{(m)}\}_{i=1}^n$ on the training set. The whitened embeddings are

\nopagebreak
\begin{equation}
\tilde z_i^{(m)} := \Sigma_m^{-1/2}\,(z_i^{(m)}-\mu_m)
\end{equation}

Alternative normalizations, such as unit-norm embeddings for cosine geometry, are possible and must be used consistently across all modality pairs and experiments.

\subsubsection*{Nested projection classes}

Fix a shared dimension $d$ after whitening. A projection family $\{\mathcal{G}_\alpha\}_{\alpha\in\mathcal{A}}$ is chosen before computing hardness, where $\mathcal{A}\subseteq\mathbb{R}_{\geq 0}$ is ordered and the family is nested in $\alpha$. The examples below are standardized choices, each with its own nesting rule.

\paragraph*{Baseline: orthogonal linear maps (Procrustes)}

\nopagebreak
\begin{equation}
\mathcal{G}_0 := \{g(x)=Qx:\; Q\in O(d)\}
\end{equation}

\paragraph*{Class 1: low-rank linear maps with operator norm bound}
For rank parameter $r$ and Lipschitz bound $L$,

\nopagebreak
\begin{equation}
\mathcal{G}^{\mathrm{lin}}_r:=\{g(x)=Wx:\; \rank(W)\le r,\ \|W\|_2\le L\}
\end{equation}

\paragraph*{Class 2: bounded-width Lipschitz MLPs}
Fix depth $D$ and width parameter $w$. Let $\mathrm{MLP}(D,w,L)$ denote ReLU networks of depth $D$ and width $w$ with spectral norm constraints chosen so the resulting global Lipschitz constant is at most $L$. Define

\nopagebreak
\begin{equation}
\mathcal{G}^{\mathrm{mlp}}_w:=\{g\in \mathrm{MLP}(D,w,L)\}
\end{equation}

\noindent
which is nested by width.

\subsubsection*{Projection hardness}

Fix an error metric on whitened embeddings; throughout we use mean-squared error:

\nopagebreak
\begin{equation}
\mathrm{err}(g; a\to b):=\frac{1}{n}\sum_{i=1}^n \|g(\tilde z_i^{(a)})-\tilde z_i^{(b)}\|^2
\end{equation}

\begin{definition}[Projection hardness]
Given a tolerance $\varepsilon\ge 0$ and a nested family $\{\mathcal{G}_\alpha\}_{\alpha\in\mathcal{A}}$, the hardness from $a$ to $b$ is

\nopagebreak
\begin{equation}
H_{a\to b}(\varepsilon):=\inf\Big\{\alpha:\ \exists g\in \mathcal{G}_\alpha \text{ such that } \mathrm{err}(g; a\to b)\le \varepsilon\Big\}
\end{equation}

\end{definition}

When a finite hardness value is used as a witness below, we assume that the infimum is attained. This holds for the finite rank and width grids used in the controlled calibrations.

\begin{remark}[Directed hardness graph]
For fixed $\varepsilon$, hardness defines a directed weighted graph on modalities. Because whitening and the map family are fixed, $H_{a\to b}(\varepsilon)$ is comparable across different pairs $(a,b)$.
\end{remark}

\subsubsection*{Relation to learning-theoretic complexity}

Projection hardness is an empirical approximation-complexity measure for a chosen normalized family. It records the smallest class index whose empirical alignment error reaches $\varepsilon$ on the fixed sample set. The population interpretation of that result also depends on generalization within the selected class.

For this population statement, the normalization maps are treated as fixed, for example by fitting them on an independent preprocessing split. To connect the two, let $(X^{(a)},X^{(b)})$ denote a fresh paired sample after the same normalization, and define the population loss

\nopagebreak
\begin{equation}
\mathcal{E}_{\mathrm{pop}}(g) := \mathbb{E}\|g(X^{(a)})-X^{(b)}\|^2
\end{equation}

Assume that, on the support of the normalized data, $\|X^{(b)}\|\le R_b$ and $\sup_{g\in\mathcal{G}_\alpha}\|g(X^{(a)})\|\le R_{g,\alpha}$. Then the squared loss is bounded by $M_\alpha=(R_b+R_{g,\alpha})^2$. Standard Rademacher uniform-convergence bounds for the bounded loss class $\ell\circ\mathcal{G}_\alpha$ give

\nopagebreak
\begin{equation}
\mathcal{E}_{\mathrm{pop}}(g)
\le
\mathcal{E}_{\mathrm{emp}}(g)
+ 2\mathfrak{R}_n(\ell\circ\mathcal{G}_\alpha)
+ 3M_\alpha\sqrt{\frac{\log(2/\delta)}{2n}}
\end{equation}

\noindent
with probability at least $1-\delta$, uniformly over $g\in\mathcal{G}_\alpha$~\cite{bartlett2002rademacher,mohri2018foundations}, where $\mathfrak{R}_n$ denotes empirical Rademacher complexity.

\begin{proposition}[Hardness and sample complexity]
\label{prop:hardness_generalization}
Fix $\varepsilon,\eta,\delta>0$ and a nested family $\{\mathcal{G}_\alpha\}$. Suppose the above boundedness condition holds and

\nopagebreak
\begin{equation}
2\mathfrak{R}_n(\ell\circ\mathcal{G}_\alpha)
+3M_\alpha\sqrt{\frac{\log(2/\delta)}{2n}}
\le \eta
\end{equation}

If the empirical hardness satisfies $H_{a\to b}(\varepsilon)\le \alpha$ and the hardness infimum is attained, then there exists $g\in\mathcal{G}_\alpha$ with population alignment error at most $\varepsilon+\eta$.
\end{proposition}

\begin{proof}
By attainment and nesting, $H_{a\to b}(\varepsilon)\le\alpha$ implies that some $g\in\mathcal{G}_\alpha$ has empirical error at most $\varepsilon$. The Rademacher bound above holds uniformly over $\mathcal{G}_\alpha$, so the same $g$ has population error at most its empirical error plus the displayed complexity term. The assumed bound on that term gives $\mathcal{E}_{\mathrm{pop}}(g)\le\varepsilon+\eta$.
\end{proof}

For a finite candidate set $\mathcal{A}_0\subset\mathcal{A}$, the conclusion holds simultaneously for every $\alpha\in\mathcal{A}_0$ by replacing $\delta$ with $\delta/|\mathcal{A}_0|$ in the displayed bound.

Projection hardness and Rademacher complexity therefore play complementary roles. Hardness identifies the first class with sufficient approximation power for the observed cross-modal relation, and the complexity term determines the sample requirement for that class to generalize. For the low-rank linear family with $\|W\|_2\le L$ and $\rank(W)\le r$, the Frobenius norm is at most $L\sqrt r$, so standard linear-class bounds grow with the rank index through a $\sqrt r$ factor, up to input and loss constants. For norm-controlled width-$w$ ReLU families, the complexity term depends on norm products and the network parameterization~\cite{bartlett2017spectrally,golowich2018size}. Accordingly, the class index reported by hardness also determines the statistical-complexity term associated with the empirical alignment result.

\subsubsection*{Projection-Parameter Sheaf and Sheaf-Laplacian Obstruction}

Hardness concerns a global map. The obstruction construction attaches local projection parameters to the vertices of the fixed site and asks whether those local choices can be made coherent across edges. Fix a declared local projection family $\mathcal{G}_{\mathrm{loc}}=\{g_w:w\in\mathbb{R}^p\}$. The family and its parameterization are held fixed across modality pairs being compared. A field $w=\{w_v\}_{v\in V}$ assigns a local projection $g_{w_v}$ to each vertex.

Neighboring parameter choices are compared by the restriction maps of the sheaf. For a general edge $e=(u,v)$, the edge discrepancy has the form

\nopagebreak
\begin{equation}
\left\|\rho_{e\leftarrow u}(w_u)-\rho_{e\leftarrow v}(w_v)\right\|^2
\end{equation}

The common-parameter construction formalized in~\autoref{def:projection_parameter_sheaf} uses $\mathbb{R}^p$ for every vertex and edge stalk, with identity restriction maps. Local projections are then compared in one shared parameter coordinate system, giving the vector-valued graph Dirichlet energy of the parameter field.

\subsubsection*{Projection-parameter sheaf}

\begin{definition}[Projection-parameter sheaf]
\label{def:projection_parameter_sheaf}
Let $G=(V,E)$ be the fixed base graph. The projection-parameter sheaf $\mathcal{P}$ on $G$ is defined by:

\begin{itemize}
\item Vertex stalks: $\mathcal{P}(v)=\mathbb{R}^p$ for all $v\in V$.
\item Edge stalks: $\mathcal{P}(e)=\mathbb{R}^p$ for all $e\in E$.
\item Restrictions: identity maps $\rho_{e\leftarrow v}=\mathrm{Id}_{\mathbb{R}^p}$ for all incidences.
\end{itemize}

\end{definition}

\noindent
For a $0$-cochain $w=\{w_v\}_{v\in V}$, the coboundary is

\nopagebreak
\begin{equation}
(\delta^0 w)(e=(u,v))=w_u-w_v
\end{equation}

\noindent
and the sheaf energy is

\nopagebreak
\begin{equation}
E_{\mathcal{P}}(w)=\sum_{(u,v)\in E}\|w_u-w_v\|^2
\end{equation}

\noindent
This energy is the $0$-Laplacian inconsistency energy of the projection-parameter sheaf. Zero energy means the local parameter choices glue over the site; on a connected site this is exactly a constant parameter field, corresponding to one global projection map.

\begin{proposition}[Global sections are constant parameter fields]
\label{prop:global_sections}
If $G$ is connected, then $\ker(\Delta^0)=\{w:\ w_v\equiv \bar w \text{ for all } v\in V\}$.
More generally, on a graph with $k$ connected components, $\ker(\Delta^0)$ is the space of parameter fields constant on each component.
\end{proposition}

\subsubsection*{Role of the sheaf formulation}

For the identity-restriction construction in~\autoref{def:projection_parameter_sheaf}, $\Delta^0=L\otimes I_p$, and the computation coincides exactly with vector-valued graph Dirichlet regularization. This baseline can therefore be computed with standard graph-regularization methods. The sheaf formulation identifies the kernel with the space of global sections, so zero energy means that the locally chosen maps glue to a global map. It also retains the same definition when neighborhoods use different parameter spaces or when parameters require transport before comparison. In general, one may take vertex stalks $\mathbb R^{p_v}$, an edge stalk $\mathbb R^{r_e}$, and incidence-specific maps $T_{e\leftarrow v}:\mathbb R^{p_v}\to\mathbb R^{r_e}$, producing the discrepancy

\nopagebreak
\begin{equation}
\left\|T_{e\leftarrow u}w_u-T_{e\leftarrow v}w_v\right\|^2
\end{equation}

When $p_u\ne p_v$, the identity penalty $\|w_u-w_v\|^2$ has no common parameter space in which to form the difference. When the dimensions agree, the same penalty cannot encode edge-dependent changes of coordinates.

\paragraph*{Concrete non-identity example.}
Suppose adjacent neighborhoods describe the same one-dimensional local projection in coordinate frames with opposite orientations. On their common edge, take $\rho_{e\leftarrow u}=+1$ and $\rho_{e\leftarrow v}=-1$. Intrinsic agreement is then $w_u=-w_v$, for which the sheaf discrepancy $(w_u+w_v)^2$ vanishes. Identity graph regularization assigns the positive penalty $(w_u-w_v)^2=4w_u^2$ to the same field. These signs, as well as higher-dimensional transport matrices, define connection-Laplacian cases of the non-identity restriction-map structure represented uniformly by a cellular sheaf. The controlled calibrations below use identity restrictions so that their obstruction values have exact analytic references. The parameter coordinates, stalk inner products, and restriction maps are part of the obstruction specification and are held fixed across modality pairs.

\subsubsection*{Obstruction invariant at fixed mean error tolerance}

Fix a loss $\ell(\cdot,\cdot)$, typically squared error on whitened embeddings. For any parameter field $w\in(\mathbb R^p)^V$, define

\nopagebreak
\begin{align}
\Err^{\mathrm{local}}_{a\to b}(w)
&:=\frac{1}{n}\sum_{v\in V}\ell\!\left(g_{w_v}(\tilde z_v^{(a)}),\tilde z_v^{(b)}\right)\\
\Var_{a\to b}(w)
&:=E_{\mathcal P}(w)
=\sum_{(u,v)\in E}\|w_u-w_v\|^2
\end{align}

\begin{definition}[Sheaf-Laplacian obstruction]
Fix a target mean-error tolerance $\varepsilon\ge 0$. Define

\nopagebreak
\begin{equation}
\label{eq:obstruction_constrained}
C_{a\to b}^{\mathcal{G}_{\mathrm{loc}}}(\varepsilon)
:=
\inf_{w\in(\mathbb R^p)^V}
\left\{
E_{\mathcal P}(w):
\frac{1}{n}\sum_{v\in V}
\ell\!\left(g_{w_v}(\tilde z_v^{(a)}),\tilde z_v^{(b)}\right)
\le\varepsilon
\right\}
\end{equation}

\end{definition}

We use the convention $\inf\varnothing=+\infty$. Equation~\eqref{eq:obstruction_constrained} optimizes over parameter fields directly. Nonunique optimizers do not alter the invariant, and its definition does not require a regularization path to recover every constrained optimum. The factor $1/n$ places its tolerance on the same mean-loss scale as projection hardness.

We suppress the superscript when the declared local family is fixed by context. Projection hardness and obstruction are therefore evaluated relative to separately declared global and local specifications.

\begin{remark}[Specification and coupling]
Projection hardness is defined relative to the declared global hierarchy $\{\mathcal{G}_\alpha\}$, whereas obstruction is defined relative to the declared local family $\mathcal{G}_{\mathrm{loc}}$ and its parameter sheaf. These choices need not coincide. Under identity restrictions, if $\mathcal{G}_{\mathrm{loc}}=\mathcal{G}_\alpha$ with the same parameterization, sample loss, and tolerance, then $H_{a\to b}(\varepsilon)\leq\alpha$ and attainment of the hardness infimum yield a constant feasible field, so $C_{a\to b}(\varepsilon)=0$. The two quantities are therefore distinct, but they are not independent under this shared-family specialization.
\end{remark}

\subsubsection*{Lagrangian computation and diagnostics}

A standard computational method is the weighted Lagrangian objective

\nopagebreak
\begin{equation}
\label{eq:sheaf_objective}
\mathcal{L}(w;\lambda)
=
\Err^{\mathrm{local}}_{a\to b}(w)
+
\lambda E_{\mathcal P}(w)
\qquad \lambda\ge 0
\end{equation}

\noindent
with a selected global minimizer

\nopagebreak
\begin{equation}
w^\ast(\lambda)\in \argmin_{w\in (\mathbb{R}^p)^V}\mathcal{L}(w;\lambda)
\end{equation}

\noindent
The constant-field benchmark is measured on the same scale,

\nopagebreak
\begin{equation}
\Err^{\mathrm{global}}_{a\to b}
:=\inf_{\theta\in \mathbb R^p}
\frac{1}{n}\sum_{v\in V}\ell\!\left(g_{\theta}(\tilde z_v^{(a)}),\tilde z_v^{(b)}\right)
\end{equation}

The pair $\big(\Err^{\mathrm{local}}_{a\to b}(w^\ast(\lambda)),\Var_{a\to b}(w^\ast(\lambda))\big)$ traces the fit--variation tradeoff. When the argmin is nonunique, any selected minimizer gives a point on this computational path. The constrained definition in~\eqref{eq:obstruction_constrained} remains unchanged, and the selection rule should be stated for reproducibility.

Equation~\eqref{eq:sheaf_objective} provides a computational scalarization of the constrained definition. Suppose the mean loss and energy are convex and lower semicontinuous in $w$, the constrained problem is feasible and attains its infimum, and a constraint qualification such as Slater's condition holds. Strong duality then recovers a constrained solution through a suitable regularization weight or an endpoint limit, up to the reciprocal scaling between the two conventional Lagrangian forms. For nonconvex neural parameterizations or inexact local optimization, the regularization path provides a computable approximation. Every returned field satisfying the mean-error constraint certifies the upper bound $C_{a\to b}(\varepsilon)\le E_{\mathcal P}(w)$.

\begin{remark}[Two failure modes]
Hardness failure occurs when the best global map in a low-complexity class cannot reach error $\varepsilon$. Obstruction failure occurs when reaching $\varepsilon$ with locally varying maps in the declared local family requires large variation energy over $G$.
\end{remark}

\autoref{fig:hardness} summarizes the two-axis diagnostic picture for the declared global hierarchy and local family. Projection hardness measures the global complexity needed to reach the target error. Sheaf-Laplacian obstruction measures the parameter variation required by locally fit maps. Low values on both axes indicate compatibility at the chosen tolerance; high values identify the corresponding failure mode. When the local family coincides with a reference hardness class, the coupling remark restricts the admissible regimes.

\begin{figure*}[!ht]
  \centering
  \includegraphics[width=0.5\textwidth]{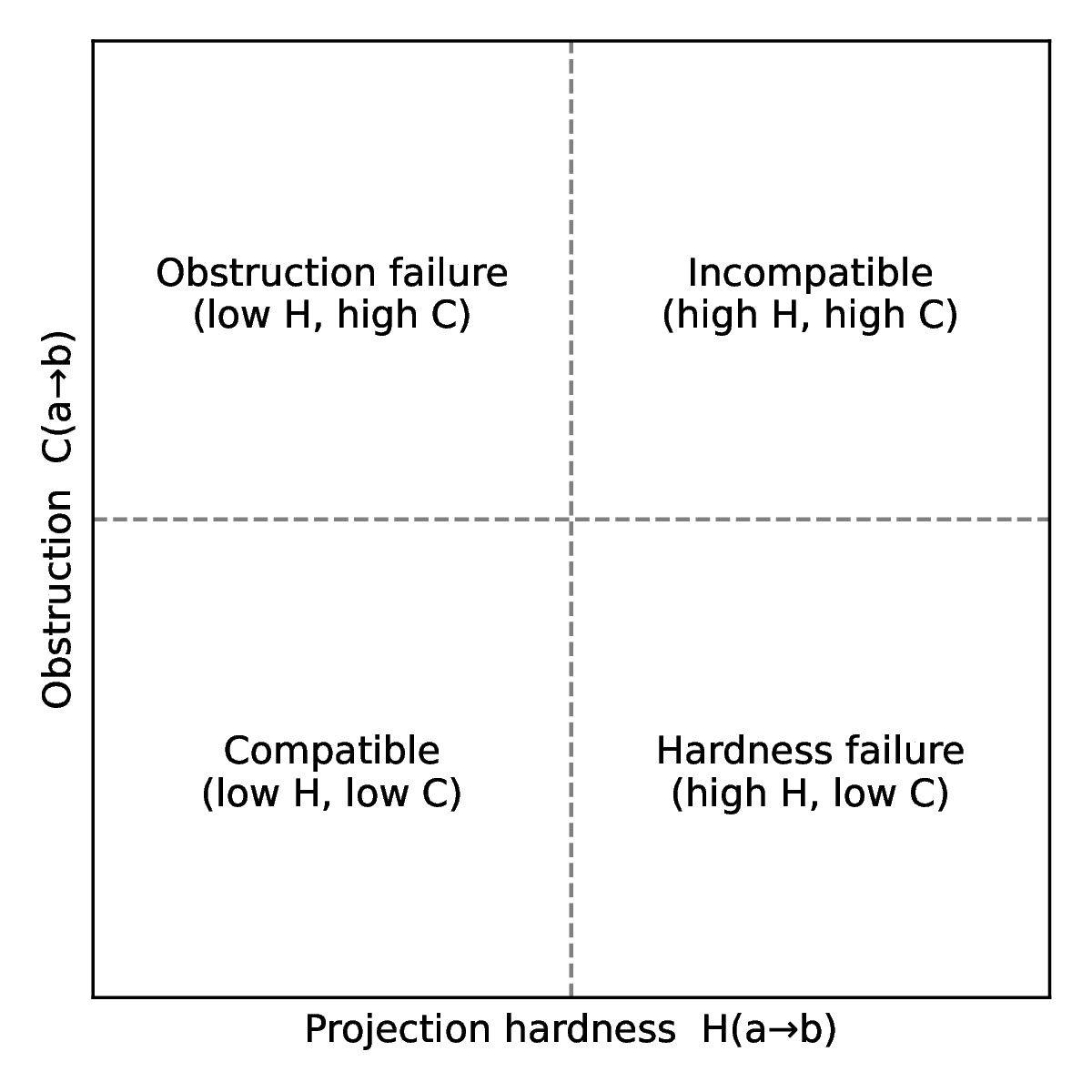}
  \caption{Diagnostic regimes defined by projection hardness and sheaf-Laplacian obstruction under declared global and local specifications.}
  \label{fig:hardness}
\end{figure*}

\autoref{fig:field} gives the parameter-field view of the obstruction invariant. A constant field corresponds to a single global map. A smoothly varying field has low obstruction. A sharp change across a cut produces large obstruction even when local fits are accurate.

\begin{figure*}[!ht]
  \centering
  \includegraphics[width=0.85\textwidth]{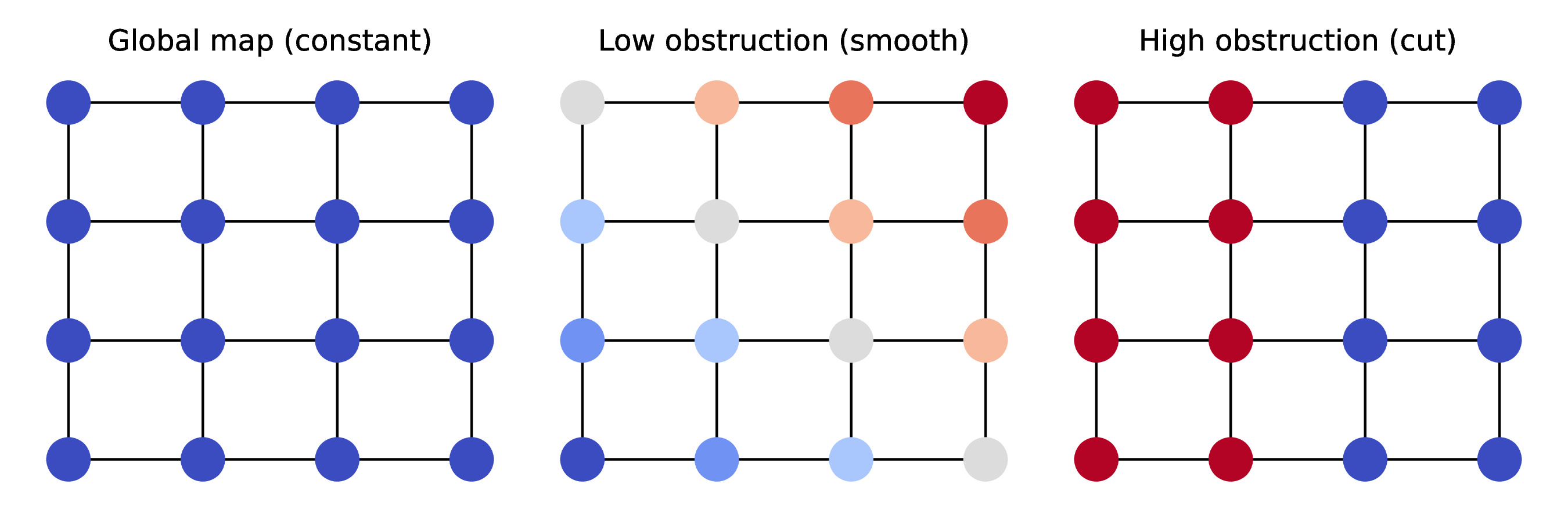}
  \caption{Parameter-field view of the sheaf-Laplacian obstruction invariant.}
  \label{fig:field}
\end{figure*}

\subsection{Spectral Gap, Stability, and a Link Between Obstruction and Global Error}

The stability results below use the identity-restriction baseline, for which $\Delta^0$ reduces to a vector-valued graph Laplacian.

\subsubsection*{Poincar\'e-type control by the spectral gap}

Let $L$ denote the (combinatorial) graph Laplacian of $G$. For vector-valued fields $w\in(\mathbb R^p)^V$, the quadratic form $\sum_{(u,v)\in E}\|w_u-w_v\|^2$ is the Dirichlet form associated with the graph Laplacian, and can be written as $\langle w,(L\otimes I_p)w\rangle$~\cite{evans1998partial}.

Let $\lambda_2(G)$ denote the second-smallest eigenvalue of $L$ (the algebraic connectivity), assuming $G$ is connected.

\begin{lemma}[Vector-valued Poincar\'e inequality on graphs]
\label{lem:poincare}
If $G$ is connected, then for any $w\in (\mathbb{R}^p)^V$ and its vertexwise mean $\bar w:=\frac{1}{n}\sum_{v\in V} w_v$,

\nopagebreak
\begin{equation}
\sum_{v\in V}\|w_v-\bar w\|^2 \;\le\; \frac{1}{\lambda_2(G)}\sum_{(u,v)\in E}\|w_u-w_v\|^2
\end{equation}

\end{lemma}

\begin{proof}
Let $0=\lambda_1<\lambda_2\le \cdots \le \lambda_n$ be the eigenvalues of $L$,
with orthonormal eigenvectors $\{u_k\}_{k=1}^n$ and
$u_1=n^{-1/2}\mathbf{1}$. For a scalar field $f\in\mathbb{R}^V$ with mean
$\bar f$, set $f_\perp=f-\bar f\mathbf{1}$. Then $f_\perp\perp\mathbf{1}$ and

\nopagebreak
\begin{equation}
\label{eq:graph_poincare_scalar}
\begin{aligned}
f_\perp
&=\sum_{k=2}^n \alpha_k u_k \\
\|f_\perp\|^2
&=\sum_{k=2}^n \alpha_k^2 \\
f_\perp^\top L f_\perp
&=\sum_{k=2}^n \lambda_k \alpha_k^2
\ge
\lambda_2(G)\sum_{k=2}^n \alpha_k^2
=
\lambda_2(G)\|f_\perp\|^2 \\
\text{Hence,}\quad \sum_{v\in V}\|w_v-\bar w\|^2
&\le
\frac{1}{\lambda_2(G)}
\sum_{(u,v)\in E}\|w_u-w_v\|^2
\end{aligned}
\end{equation}

\end{proof}

\subsubsection*{From small obstruction to near-global parameter consistency}

The Poincar\'e inequality states that small variation energy forces $w$ to be close to a constant parameter field, provided $\lambda_2(G)$ is bounded away from zero. To translate this into alignment consequences we assume mild regularity of the loss with respect to parameters.

\begin{definition}[Parameter-Lipschitz loss]
The per-sample loss $\ell(g_w(x),y)$ is $L_w$-Lipschitz in $w$ (uniformly over relevant $x,y$) if

\nopagebreak
\begin{equation}
\Big|\ell(g_{w}(x),y)-\ell(g_{w'}(x),y)\Big|\le L_w\|w-w'\|
\end{equation}

\noindent
for all $w,w',x,y$ under consideration.
\end{definition}

This is a modeling assumption; in practice it can be enforced or approximated by parameter norm control and Lipschitz constraints on $g_w$.

\begin{theorem}[Obstruction controls excess global-map error]
\label{thm:obstruction_to_global}
Assume $G$ is connected and the per-sample loss is $L_w$-Lipschitz in $w$.
Let $w=\{w_v\}_{v\in V}$ be any parameter field and let $\bar w$ be its mean.
Then

\nopagebreak
\begin{equation}
\label{eq:obstruction_to_global}
\begin{aligned}
\frac{1}{n}\sum_{v\in V}
\ell\!\left(g_{\bar w}(\tilde z_v^{(a)}), \tilde z_v^{(b)}\right)
&\le
\frac{1}{n}\sum_{v\in V}
\ell\!\left(g_{w_v}(\tilde z_v^{(a)}), \tilde z_v^{(b)}\right) \\
&\quad+
L_w
\sqrt{
\frac{1}{n\lambda_2(G)}
\sum_{(u,v)\in E}
\|w_u-w_v\|^2
}
\end{aligned}
\end{equation}

\end{theorem}

\begin{proof}
By Lipschitzness,

\vspace{1em}
\nopagebreak
\begin{equation}
\ell(g_{\bar w}(x_v),y_v)\le \ell(g_{w_v}(x_v),y_v)+L_w\|\bar w-w_v\|
\end{equation}

\vspace{1em}
\noindent
Averaging over $v$ and applying Cauchy–Schwarz:

\vspace{1em}
\nopagebreak
\begin{equation}
\frac{1}{n}\sum_v \|\bar w-w_v\|
\le
\sqrt{\frac{1}{n}\sum_v\|\bar w-w_v\|^2}
\end{equation}

\vspace{1em}
\noindent
Now apply \autoref{lem:poincare} to bound $\sum_v\|\bar w-w_v\|^2$ by $E_{\mathcal P}(w)/\lambda_2(G)$.
\end{proof}

\begin{remark}[Interpretation]
\autoref{thm:obstruction_to_global} shows that if a locally varying fit achieves small local error and small variation energy, then a single global parameter vector (the mean) achieves nearly the same error, with an explicit degradation controlled by the spectral gap $\lambda_2(G)$. In particular, if $C_{a\to b}(\varepsilon)<\infty$, applying the theorem to a feasible sequence approaching the infimum in~\eqref{eq:obstruction_constrained} gives

\nopagebreak
\begin{equation}
\Err^{\mathrm{global}}_{a\to b}
\le
\varepsilon+L_w\sqrt{\frac{C_{a\to b}(\varepsilon)}{n\lambda_2(G)}}
\end{equation}

Thus small obstruction certifies that a feasible local field can be replaced by a single global parameter with controlled excess error within the declared local family.
\end{remark}

\subsection{Compatibility Profiles and Directed Non-Transitivity}

The framework gives a compatibility profile for each directed pair $(a\to b)$ at tolerance $\varepsilon$:

\nopagebreak
\begin{equation}
\big(H_{a\to b}(\varepsilon),\ C_{a\to b}(\varepsilon)\big)
\end{equation}

Thresholding these quantities induces binary relations (``compatible'' vs.\ ``incompatible''). The directed weighted profile remains the primary object.

\subsubsection*{A minimal compatibility relation}
Fix thresholds $\alpha_0$ and $\tau_0$. Define

\nopagebreak
\begin{equation}
a \xRightarrow[\varepsilon]{(\alpha_0,\tau_0)} b
\iff
H_{a\to b}(\varepsilon)\le \alpha_0\ \ \text{and}\ \ C_{a\to b}(\varepsilon)\le \tau_0
\end{equation}

When $\mathcal{G}_{\mathrm{loc}}=\mathcal{G}_{\alpha_0}$ under the assumptions of the coupling remark, the obstruction threshold is redundant for pairs satisfying the hardness condition. Otherwise, the two thresholds refer to the separately declared global and local specifications.

Even with this restrictive definition, compatibility can fail to be transitive. \autoref{fig:bridge} illustrates the mechanism used below. An intermediate modality can make the two stages $a\to c$ and $c\to b$ low-complexity even when the direct map $a\to b$ requires a much larger family.

\begin{figure*}[!ht]
  \centering
  \includegraphics[width=0.7\textwidth]{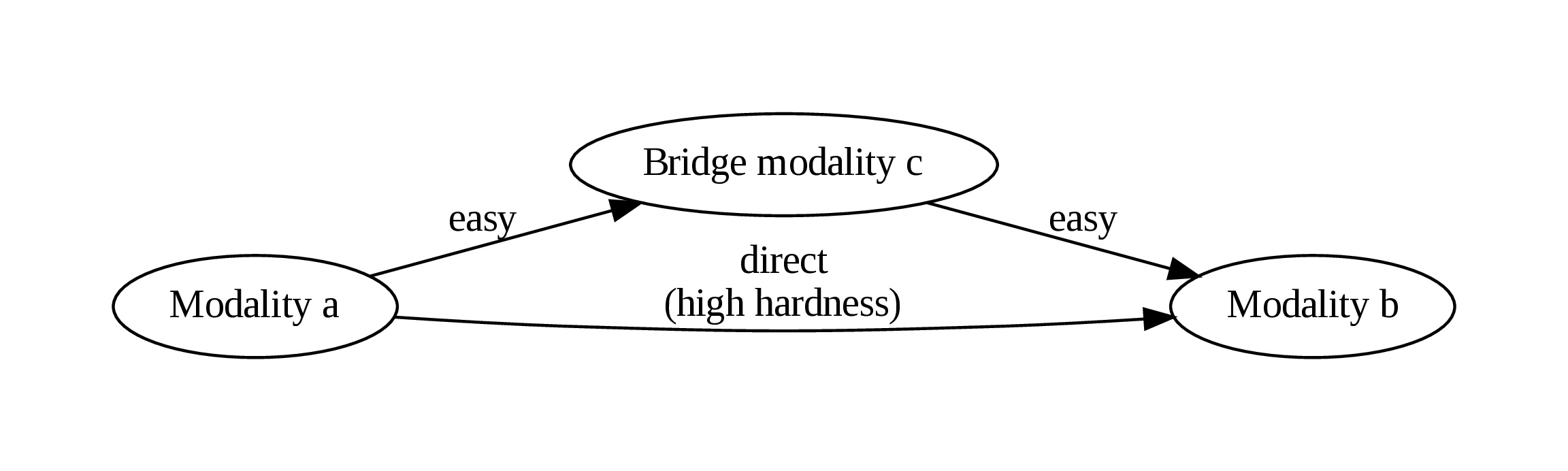}
    \caption{Bridging via an intermediate modality.}
  \label{fig:bridge}
\end{figure*}

\subsection{Bridging via composed projection families}

Let $a,b,c$ be three modalities. Consider a two-stage alignment through $c$. For a family $\mathcal{G}$, define the composed family

\nopagebreak
\begin{equation}
\mathcal{G}_{a\to c\to b}:=\{g_{c\to b}\circ g_{a\to c}:\ g_{a\to c}\in\mathcal{G},\ g_{c\to b}\in\mathcal{G}\}
\end{equation}

\noindent
Because the family can depend on a complexity index, we make this explicit.

\begin{definition}[Composed hardness]
Fix a tolerance $\varepsilon$. The two-stage hardness is

\nopagebreak
\begin{equation}
\begin{aligned}
H_{a\to c\to b}(\varepsilon)
:=
\inf\Big\{\alpha: \exists g_{a\to c}\in \mathcal{G}_\alpha,\ \exists g_{c\to b}\in \mathcal{G}_\alpha \\
\xRightarrow \mathrm{err}(g_{c\to b}\circ g_{a\to c};a\to b)\le \varepsilon\Big\}
\end{aligned}
\end{equation}

\end{definition}

Direct and stagewise obstruction values below use declared local families on the same site, with the same parameter coordinates, stalk inner products, and restriction conventions wherever values are compared.

\begin{definition}[Stagewise obstruction]
Let $w^{a\to c}$ and $w^{c\to b}$ be parameter fields for the two stages, with projection-parameter sheaves $\mathcal P_{a\to c}$ and $\mathcal P_{c\to b}$ on the same site. Define

\nopagebreak
\begin{equation}
\begin{aligned}
C_{a\to c\to b}(\varepsilon)
:=\inf_{w^{a\to c},w^{c\to b}}
\Bigg\{&E_{\mathcal P_{a\to c}}(w^{a\to c})
+E_{\mathcal P_{c\to b}}(w^{c\to b}):\\
&\frac{1}{n}\sum_{v\in V}\ell\!\left(
g_{w^{c\to b}_v}
\!\left(g_{w^{a\to c}_v}(\tilde z_v^{(a)})\right),
\tilde z_v^{(b)}
\right)\le\varepsilon\Bigg\}
\end{aligned}
\end{equation}

\end{definition}

As in the direct case, a practical scalarization jointly minimizes the mean composed loss plus $\lambda_1E_{\mathcal P_{a\to c}}+\lambda_2E_{\mathcal P_{c\to b}}$. It is exact under the corresponding convexity and duality conditions; in nonconvex settings, every feasible returned pair provides a computable upper bound on the constrained stagewise obstruction.

\begin{definition}[Hardness bridge]
Given a comparison factor $\eta_H\in(0,1)$, modality $c$ is an $\eta_H$-hardness bridge for $(a,b)$ at tolerance $\varepsilon$ if

\nopagebreak
\begin{equation}
H_{a\to c\to b}(\varepsilon)\le \eta_H H_{a\to b}(\varepsilon)
\end{equation}

\end{definition}

\begin{definition}[Full bridge]
Given comparison factors $\eta_H,\eta_C\in(0,1)$, modality $c$ is an $(\eta_H,\eta_C)$-full bridge for $(a,b)$ at tolerance $\varepsilon$ if

\nopagebreak
\begin{equation}
H_{a\to c\to b}(\varepsilon)\le \eta_H H_{a\to b}(\varepsilon)
\qquad
C_{a\to c\to b}(\varepsilon)\le \eta_C C_{a\to b}(\varepsilon)
\end{equation}

\end{definition}

\begin{remark}
The obstruction component is stagewise: one measures the variation needed to fit $a\to c$ and $c\to b$ locally, each over the same site, then compares the summed variation with the direct $a\to b$ obstruction. The controlled bridge calibration in \autoref{fig:exp-projection-hardness-calibration} tests the hardness component.
\end{remark}

\subsubsection*{A Concrete Non-Transitivity Construction in a ReLU Setting}
\label{sec:relu_construction}

We adapt this triangle-wave composition mechanism to a three-modality setting. The resulting bounded-width construction has a two-stage map with lower required width than the direct map at the same depth: composition increases effective depth and representational complexity while preserving the width of each stage.

\subsubsection*{A width lower bound via breakpoints in one dimension}

Consider one-hidden-layer ReLU networks on $\mathbb{R}$:

\nopagebreak
\begin{equation}
f(x)=\beta_0+\beta_1x+\sum_{j=1}^w a_j \sigma(b_j x + c_j)\qquad \sigma(t)=\max\{t,0\}
\end{equation}

\noindent
The affine term introduces no breakpoint and can alternatively be implemented by at most two additional ReLU units, so it does not affect the asymptotic width separation. Such functions are continuous piecewise-linear with at most $w$ breakpoints (points where the slope changes).

\begin{lemma}[Breakpoint bound]
\label{lem:breakpoints}
Any one-hidden-layer ReLU network of width $w$ on $\mathbb{R}$ has at most $w$ breakpoints, hence at most $w + 1$ linear pieces.
\end{lemma}

\begin{proof}
Each unit $\sigma(b_j x+c_j)$ changes from linear to linear at the single point $x=-c_j/b_j$ (when $b_j\neq 0$). A linear combination of such units can only change slope at points where at least one unit changes regime. Thus the set of possible slope-change points is contained in the set of at most $w$ thresholds.
\end{proof}

Define $\mathcal{G}^{\mathrm{mlp}}_w$ to be one-hidden-layer ReLU networks of width $w$ with any fixed Lipschitz bound $L\ge1$ (the Lipschitz constraint does not affect the breakpoint counting argument).

\subsubsection*{Composition multiplies the number of linear regions}

\begin{lemma}[Composing width-$w$ one-hidden-layer networks increases pieces]
\label{lem:composition_pieces}
Let $g,h\in \mathcal{G}^{\mathrm{mlp}}_w$ be one-hidden-layer ReLU networks on $\mathbb{R}$. Then the composition $h\circ g$ is continuous piecewise-linear and can have $\Omega(w^2)$ linear pieces in the worst case.
\end{lemma}

\begin{proof}
For an integer $K\ge1$, let $T_K:[0,1]\to[0,1]$ be the triangular folding map obtained by linearly interpolating

\vspace{1em}
\nopagebreak
\begin{equation}
T_K\!\left(\frac{j}{2K}\right)
=
\begin{cases}
0 & j\text{ even}\\
1 & j\text{ odd}
\end{cases}
\qquad j=0,\ldots,2K
\end{equation}

\vspace{1em}
\noindent
It has $2K$ affine pieces, $2K-1$ internal breakpoints, and slopes $\pm2K$. On the relevant intervals define

\nopagebreak
\begin{equation}
g_K(x):=\frac{T_K(x)}{2K}
\qquad
h_K(y):=\frac{T_K(2Ky)}{4K^2}
\end{equation}

\vspace{1em}
\noindent
and extend both linearly beyond those intervals using their endpoint slopes. Each map is $1$-Lipschitz and has $2K-1$ internal breakpoints, hence admits the affine-plus-hinges representation above with width $2K-1$. Their composition on $[0,1]$ is

\vspace{1em}
\nopagebreak
\begin{equation}
(h_K\circ g_K)(x)=\frac{T_K(T_K(x))}{4K^2}
\end{equation}

\vspace{1em}
On each of the $2K$ monotone pieces of the inner $T_K$, its image traverses $[0,1]$; the $2K$ pieces of the outer $T_K$ therefore pull back to $2K$ distinct pieces there. Consequently $T_K\circ T_K$ has exactly $4K^2$ affine pieces. Taking $K=\lfloor(w+1)/2\rfloor$ gives $g_K,h_K\in\mathcal G_w^{\mathrm{mlp}}$ with a fixed Lipschitz bound and $\Omega(w^2)$ pieces in the composition.
\end{proof}

\subsubsection*{A bridging construction}

For this construction we use identity normalization on the displayed scalar coordinates.

Use the maps $g_K,h_K$ from \autoref{lem:composition_pieces}. Let modality $a$ have embedding $\tilde z^{(a)}_i=x_i\in[0,1]$, modality $c$ have embedding $\tilde z^{(c)}_i=g_K(x_i)$, and modality $b$ have embedding $\tilde z^{(b)}_i=(h_K\circ g_K)(x_i)$. Take the sample set to include the alternating vertices of the $4K^2$ affine pieces of the composition, so exact interpolation must reproduce each successive change of slope. Writing $w=2K-1$ for the stage width, we obtain:

\begin{itemize}
\item $a\to c$ is realizable by $g_K$ with width $w$, so $H_{a\to c}(0)=O(w)$.
\item $c\to b$ is realizable by $h_K$ with width $w$, so $H_{c\to b}(0)=O(w)$.
\item The sampled values of $h_K\circ g_K$ alternate at the vertices of its $4K^2$ pieces. Any continuous piecewise-linear interpolant of those values must change slope between successive rises and falls, and hence has $\Omega(K^2)=\Omega(w^2)$ breakpoints. By \autoref{lem:breakpoints}, exact direct alignment at the same depth requires width $\Omega(w^2)$.
\end{itemize}

\noindent
This yields a strict, explicit bridging effect at fixed depth: the two-stage hardness can be $O(w)$ and the direct hardness is $\Omega(w^2)$.

\begin{theorem}[Explicit non-transitivity via bridging]
\label{thm:nontransitive}
In the setting above (one-dimensional embeddings, one-hidden-layer ReLU family, and the stated noiseless sample grid), there exist modalities $a,b,c$ for which

\nopagebreak
\begin{equation}
H_{a\to c}(0)=O(w)\qquad
H_{c\to b}(0)=O(w)\qquad
H_{a\to b}(0)=\Omega(w^2)
\end{equation}

\noindent
Consequently, any thresholded ``compatibility'' relation based on hardness at scale $O(w)$ is non-transitive.
\end{theorem}

\begin{remark}
The construction underlying \autoref{thm:nontransitive} uses noiseless samples and a fixed site to isolate the effect of representational complexity under the chosen projection family.
\end{remark}

\subsubsection*{Obstruction-Driven Non-Transitivity}

Hardness is one source of alignment failure. A second source arises when good local maps exist throughout the site and their gluing within the declared local family requires large spatial variation over the semantic neighborhood graph.

\subsubsection*{A two-cluster sign-flip toy model}

Let $G$ be a graph whose vertices split into two well-connected subgraphs $V_+$ and $V_-$ with a sparse cut between them. Suppose the best alignment between modalities differs by a sign flip between the clusters:

\nopagebreak
\begin{equation}
\tilde z_v^{(b)} \approx
\begin{cases}
\tilde z_v^{(a)} & v\in V_+\\
-\tilde z_v^{(a)} & v\in V_-
\end{cases}
\end{equation}

If the projection family includes scalar multiplication by $\pm 1$, each cluster admits a perfect local map. Every globally constant map incurs error on one cluster. A locally varying parameter field attains near-zero local error by setting $w_v=+1$ on $V_+$ and $w_v=-1$ on $V_-$. The obstruction energy then concentrates on edges crossing the cut.

\begin{proposition}[Cut-induced obstruction]
In the sign-flip model with scalar parameters $w_v\in\mathbb{R}$, the minimum variation energy among perfect-fitting fields satisfies

\nopagebreak
\begin{equation}
\min\Big\{\sum_{(u,v)\in E}(w_u-w_v)^2: w_v=+1\ \forall v\in V_+,\ w_v=-1\ \forall v\in V_-\Big\}
= 4\,|E(V_+,V_-)|
\end{equation}

\noindent
where $E(V_+,V_-)$ is the set of cut edges.
\end{proposition}

\begin{remark}
This shows that obstruction can be large even when local fits are perfect. Its dependence on the site geometry (cut size) makes a fixed, modality-independent site necessary for comparability.
\end{remark}

\subsection{Assumptions Relating Site Geometry to Latent Semantics}

The interpretation of the invariants depends on how closely the fixed site $G$ represents semantic locality. We state the assumptions used to connect the formal quantities to the latent domain.

\begin{definition}[Latent locality approximation (informal)]
We say $G$ approximates latent locality if vertices connected by an edge correspond to samples that are near in latent semantic distance and latent-near samples are likely connected by short paths.
\end{definition}

A typical synthetic regime constructs $G$ directly from ground-truth latent distances, making the approximation exact by design. In general settings, consensus or label-based graphs aim to approximate this condition.

\begin{proposition}[Latent interpretation under site assumptions (informal)]
Assume: $G$ approximates latent locality, each modality embedding $\tilde z^{(m)}$ is locally regular on $G$ (e.g., Lipschitz along edges), and the global hierarchy and declared local family are Lipschitz-controlled and sufficiently expressive for their stated roles.

\nopagebreak
\noindent
\\Then:

\begin{itemize}
\item low hardness suggests the existence of a simple global alignment between the modalities on the dataset,
\item low obstruction suggests that locally optimal alignments within the declared local family glue coherently across neighborhoods,
\item large finite obstruction suggests region-dependent local correspondences within the declared local family after global expressivity has been controlled by the global hierarchy.
\end{itemize}

\end{proposition}

Under these assumptions, hardness and obstruction describe two distinct properties of the latent cross-modal relation. Applications may specialize the assumptions according to the available semantic graph and projection family.

\section{Controlled Calibration Experiments}

We evaluate the proposed quantities in controlled settings with analytically known reference behavior. This allows direct verification of the projection-hardness estimator, the scaling of constrained obstruction at zero error, and the dependence of obstruction on the fixed site graph $G$. In the scalar obstruction model, the unique zero-error field is also the exact $\lambda=0$ solution of~\eqref{eq:sheaf_objective}, linking the analytic calibration to the computational objective.

Throughout this section we use identity normalization on the displayed scalar coordinates for every modality and condition.

\subsection{Projection-hardness calibration with a known bridge}

The first experiment evaluates projection hardness in a known-answer bridge setting. For each $K \in \{2,4,8,16\}$, let $T_K:[0,1]\to[0,1]$ denote the $K$-period triangular folding map from \autoref{lem:composition_pieces}. We use an equally spaced training grid of $4096$ points and an equally spaced evaluation grid of $8192$ points, both including the endpoints of $[0,1]$. On either grid the modalities are

\nopagebreak
\begin{equation}
a_i=x_i\qquad c_i=T_K(x_i)\qquad b_i=T_K(c_i)=T_K(T_K(x_i))
\end{equation}

The intermediate modality $c$ therefore defines an exact bridge between $a$ and $b$. Since $T_K$ has $2K-1$ internal breakpoints, each stage $a\to c$ and $c\to b$ is representable by a one-hidden-layer ReLU spline with $2K-1$ knots. The direct composition $T_K\circ T_K$ has $4K^2-1$ internal breakpoints. The resulting separation between staged and direct alignment is known analytically.

\noindent
\\We use nested fixed-knot ReLU spline families,

\nopagebreak
\begin{equation}
\mathcal{G}_w=\left\{x\mapsto \beta_0+\beta_1x+\sum_{j=1}^{w}\gamma_j\operatorname{ReLU}(x-\tau_j)\right\}
\end{equation}

\noindent
where the knots $\tau_j$ follow a fixed nested dyadic sequence: level $r$ appends the odd points $j/2^r$, and the candidate-width grid is

\nopagebreak
\begin{equation}
w\in\{1,3,7,15,31,63,127,255,511,1023\}
\end{equation}

Thus the first $2^r-1$ knots are the full level-$r$ dyadic grid without endpoints, and $\mathcal{G}_w\subseteq\mathcal{G}_{w'}$ whenever $w\leq w'$. For each width, all coefficients are fit by ridge-stabilized normal-equation least squares with ridge coefficient $10^{-12}$. Error is mean-squared error on the $8192$-point evaluation grid, and the target tolerance is $\varepsilon=10^{-8}$.

This calibration uses width as its declared complexity index and isolates the structural width separation; no additional Lipschitz constraint is imposed on the fitted spline coefficients.

\noindent
\\The empirical hardness estimate is

\nopagebreak
\begin{equation}
\widehat{H}_{p\to q}(\varepsilon)=\min\left\{w:\operatorname{MSE}_w(p\to q)\leq\varepsilon\right\}
\end{equation}

For the staged map $a\to c\to b$, the same width $w$ is used for both stages. Stage one is fit from $a_{\mathrm{train}}$ to $c_{\mathrm{train}}$, stage two from $c_{\mathrm{train}}$ to $b_{\mathrm{train}}$, and the composed prediction evaluates stage two at the stage-one predictions on the evaluation grid. If no tested width reaches $\varepsilon$, the result is recorded as the right-censored observation $\widehat H>1023$.

Two control intermediates test whether the measured bridge advantage depends on sample correspondence. For each $K$ and each $s\in\{0,\ldots,9\}$, NumPy's \texttt{default\_rng} uses seed $100000K+s$. In the shuffled control, the generator randomly permutes $c_{\mathrm{train}}$, preserving its marginal values and destroying sample correspondence. In the random control, the generator draws $c_{\mathrm{train}}$ independently from $\operatorname{Uniform}[0,1]$. The target $b$ and fitting procedure remain unchanged. The true bridge is deterministic and is evaluated once per $K$; each control has $10$ repetitions per $K$.

\autoref{fig:exp-projection-hardness-calibration} summarizes the calibration. Panel A compares the empirical hardness of the direct map $a\to b$ with the staged map $a\to c\to b$ as $K$ increases. The analytic breakpoint counts are $2K-1$ for each stage and $4K^2-1$ for the direct composition, so the direct map is expected to require substantially larger width at the same target error. Right-censored observations are placed at $2\cdot1023=2046$ for display on the finite log scale. Panel B uses this plotting value, denoted $\widehat H^{\mathrm{plot}}$, and reports

\nopagebreak
\begin{equation}
\log_2\!\left(\widehat{H}^{\mathrm{plot}}_{\mathrm{direct}}/
\widehat{H}^{\mathrm{plot}}_{\mathrm{staged}}\right)
\end{equation}

When both values are observed, a positive ratio means that the staged map reaches the target error at lower width. A value involving $2046$ records a censored comparison at the tested-width boundary. The true intermediate shows the expected staged advantage. The shuffled and random controls do not preserve the samplewise bridge relation and do not produce the same systematic advantage.

\begin{figure}[H]
    \centering
    \includegraphics[width=\linewidth]{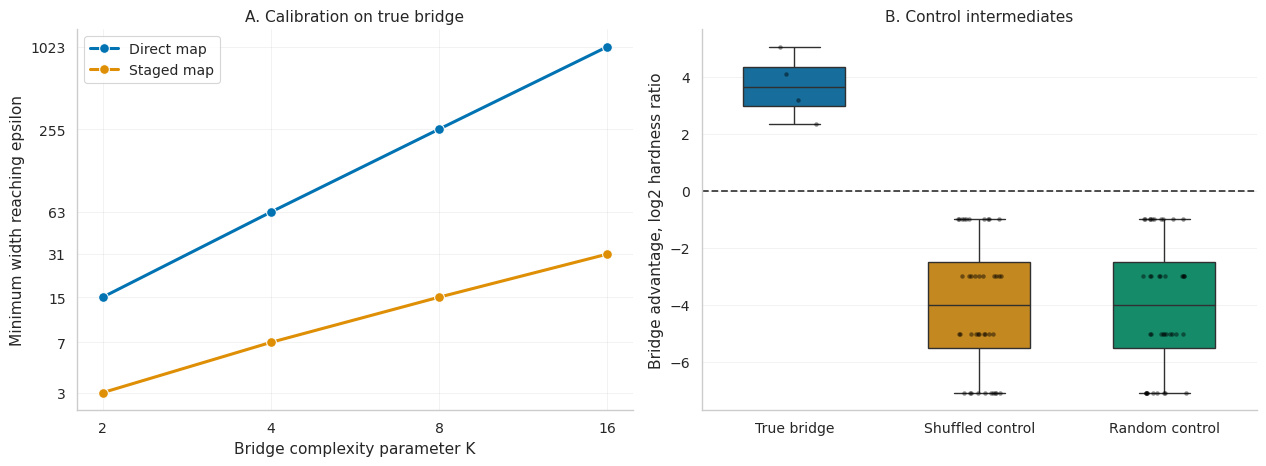}
    \caption{Projection-hardness calibration in the known bridge experiment. Panel A compares direct and staged hardness as bridge complexity increases; panel B compares the true intermediate with shuffled and random controls. Right-censored values are displayed at $2046$.}
    \label{fig:exp-projection-hardness-calibration}
\end{figure}

\subsection{Obstruction calibration in a scalar sign-flip model}

The second experiment evaluates the obstruction energy in a controlled sign-flip construction with a known analytic reference. Each graph has two clusters of $80$ vertices. For each seed $s\in\{0,\ldots,39\}$, latent coordinates are sampled uniformly on $[0,0.45]$ and $[0.55,1]$ for the two clusters and sorted within each cluster. Each vertex is joined to the preceding and following three positions when available, giving $468$ undirected within-cluster edges and connected cluster subgraphs. For each

\nopagebreak
\begin{equation}
q\in\{1,2,4,8,16,32,64\}
\end{equation}

\noindent
exactly $q$ distinct edges are then sampled uniformly without replacement from all cross-cluster pairs. The graph is unweighted and has $468+q$ edges. NumPy's \texttt{default\_rng} is initialized with $s$ for graph construction. We use scalar local maps

\nopagebreak
\begin{equation}
g_w(x)=wx
\end{equation}

\noindent
In the planted condition, the target relation is

\nopagebreak
\begin{equation}
b_i=\begin{cases}a_i & i\in V_+\\ -a_i & i\in V_-\end{cases}
\end{equation}

The zero-error local parameter field is therefore $w_i=+1$ on $V_+$ and $w_i=-1$ on $V_-$. Each cut edge contributes $(1-(-1))^2=4$ to the sheaf-Laplacian energy. The expected obstruction is

\nopagebreak
\begin{equation}
C_0=4|E(V_+,V_-)|
\end{equation}

For any nonzero scalar inputs $a_i$, setting $b_i=w_i a_i$ makes the zero-error local field uniquely determined by $w_i=b_i/a_i$. Because the energy is independent of the input magnitudes once this field is fixed, the constrained obstruction at $\varepsilon=0$ is its graph Dirichlet energy,

\nopagebreak
\begin{equation}
C_0=\sum_{(i,j)\in E}(w_i-w_j)^2
\end{equation}

The closed-form field also identifies the exact zero-error endpoint of the computational objective. With squared loss, $A=\operatorname{diag}(a_1,\ldots,a_n)$, target vector $b$, and graph Laplacian $L$, the first-order condition for~\eqref{eq:sheaf_objective} is

\nopagebreak
\begin{equation}
(A^2+n\lambda L)w=Ab
\end{equation}

At $\lambda=0$, nonzero inputs give the unique solution $w=A^{-1}b$, exactly the field evaluated in the calibration. The reported zero-error curve is therefore both the constrained optimum at $\varepsilon=0$ and the exact $\lambda=0$ endpoint of the Lagrangian path.

We compare the planted sign-flip condition with two controls. The compatible control uses $b_i=a_i$ for all samples and has zero obstruction. The shuffled-sign control contains exactly $80$ signs of each value, randomly permuted independently of the planted partition with seed $10000+100q+s$. This control tests whether high obstruction can arise from local disagreement unrelated to the planted structure. Such obstruction indicates parameter variation relative to $G$; its semantic interpretation requires a justified site.

\autoref{fig:exp-obstruction-calibration} reports the exact energy calculation over the $40$ graph draws at each $q$. Panel A fixes $q=16$ and compares the compatible control, the planted sign-flip construction, and the shuffled-sign control after normalization by the planted analytic reference $4q$. Panel B varies $q$ in the planted construction and compares the mean computed energy with the analytic reference. The planted value equals $4q$ in every draw.

\begin{figure}[H]
    \centering
    \includegraphics[width=\linewidth]{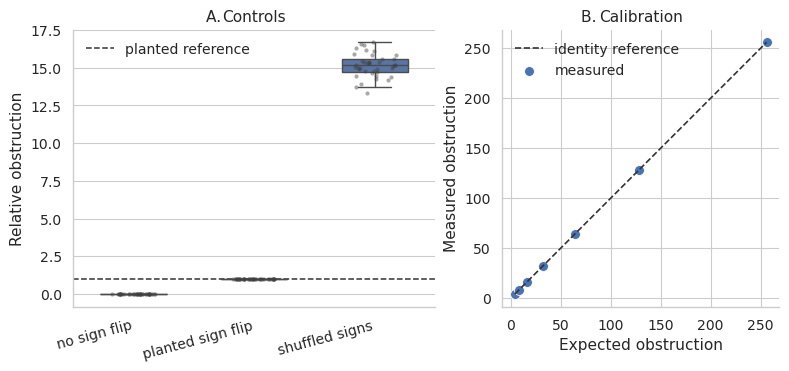}
    \caption{Obstruction calibration for the scalar sign-flip construction. Panel A compares compatible, planted, and shuffled-sign conditions; panel B compares measured obstruction with the analytic reference across cut sizes.}
    \label{fig:exp-obstruction-calibration}
\end{figure}

\subsection{Site-sensitivity calibration}

The third experiment evaluates the dependence of obstruction on the fixed site graph $G$. It uses $160$ vertices split into two clusters of $80$, the zero-error tolerance $\varepsilon=0$, and $50$ seeds $s\in\{0,\ldots,49\}$. Scalar inputs are drawn as $a_i\sim\operatorname{Uniform}[0.5,1.5]$ using \texttt{default\_rng} seed $s$. The data, projection family, tolerance, and total number of graph edges remain fixed as edge placement in $G$ varies.

We reuse the scalar projection family $g_w(x)=wx$ and construct three conditions. In the structured sign-flip condition, the required local parameter is $w_i=+1$ on one latent cluster and $w_i=-1$ on the other. In the compatible control, $w_i=+1$ for all samples. In the random-sign control, signs are assigned independently of the latent clusters.

\noindent
\\Each cluster first receives its $79$-edge chain, which guarantees connectedness. For

\nopagebreak
\begin{equation}
q\in\{1,2,4,8,16,32,64,128,256\}
\end{equation}

\noindent
$q$ cross-cluster edges are sampled uniformly without replacement, and within-cluster edges are sampled uniformly without replacement from the remaining pairs until the unweighted graph has exactly $900$ edges. The graph generator uses seed $10000s+q$. Thus increasing $q$ replaces within-cluster edges with cross-cluster edges without changing the edge count. Because perfect local fitting gives $w_i=b_i/a_i$, the obstruction is computed as

\nopagebreak
\begin{equation}
C_0(G)=\sum_{(i,j)\in E(G)}(w_i-w_j)^2
\end{equation}

In the structured sign-flip condition, each cross-cluster edge contributes $4$, so the analytic reference is

\nopagebreak
\begin{equation}
C_0(G_q)=4q
\end{equation}

The compatible control tests that the computation does not introduce spurious obstruction. The random-sign control draws signs independently and uniformly from $\{-1,+1\}$ with seed $50000+s$ (with an all-equal draw deterministically flipped at one vertex). The same sign realization is reused across the $q$ values for a given $s$, isolating the effect of graph placement. This control tests whether high obstruction can arise from local disagreement unrelated to the latent partition.

\autoref{fig:exp-site-sensitivity} shows obstruction as the median over the $50$ random graph draws, with bands indicating the interquartile range. All $1350$ field-graph evaluations have zero local mean-squared error. As $q$ increases, within-cluster edges are replaced by cross-cluster edges with the total number of edges fixed. The structured sign-flip condition equals the analytic reference $4q$ in every draw. The compatible control is exactly zero because a constant parameter field fits all samples. The random-sign control can also produce large obstruction, illustrating why obstruction must be interpreted relative to the chosen site $G$ and the hypothesis that its edges encode meaningful locality.

\begin{figure}[H]
    \centering
    \includegraphics[width=\linewidth]{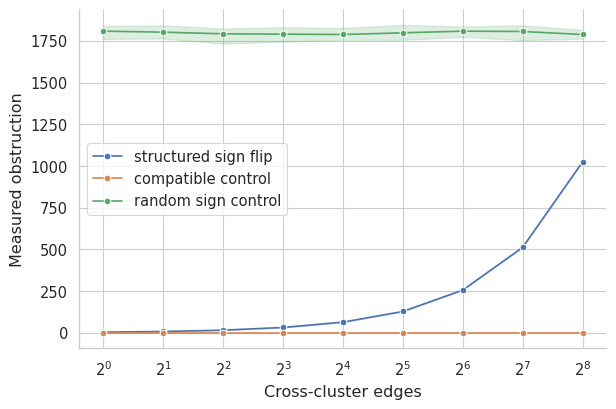}
    \caption{Site-sensitivity calibration of the obstruction quantity under varying cross-cluster edge placement. The data and projection family are fixed.}
    \label{fig:exp-site-sensitivity}
\end{figure}

\subsection{Practical case-analysis protocol}

The same computation applies to small paired multimodal datasets. For example, in an image-text collection with paired captions, one may take frozen image and text encoders such as CLIP-style encoders~\cite{radford2021learning}, build a fixed site $G$ from metadata labels, object identities, temporal adjacency, or a predeclared consensus graph, and then compute all modality-pair quantities on that site.

\noindent
\\A complete report should:

\begin{enumerate}
\item state the site construction and basic graph statistics, including edge density, component count, weights, and $\lambda_2(G)$;
\item report projection hardness for all modality pairs under a single normalization;
\item report the constrained sheaf-Laplacian obstruction in~\eqref{eq:obstruction_constrained}, including whether it was solved directly or approximated through~\eqref{eq:sheaf_objective};
\item repeat the computation under a small set of graph perturbations to determine whether high obstruction is stable across reasonable site choices.
\end{enumerate}

For a Lagrangian approximation, the weight grid, optimizer, initialization, stopping rule, feasibility tolerance, and handling of nonunique solutions should also be reported. The controlled experiments above apply the graph-relative part of this protocol in settings where the expected answer is known analytically.

\subsection{Summary of calibration results}

Together, the three experiments verify the expected behavior of the proposed quantities under controlled conditions. The bridge experiment recovers the analytically constructed separation between staged and direct alignment. Shuffled and random intermediates do not yield the same systematic advantage. The obstruction calibration reproduces the expected cut-induced energy in a scalar sign-flip model and identifies it with the exact $\lambda=0$ endpoint of the Lagrangian objective. The site-sensitivity experiment shows how the same constrained energy changes with $G$ when the data and projection family remain unchanged. These results establish computability in the tested settings and clarify obstruction as a graph-relative local-to-global inconsistency measure. The practical protocol gives the corresponding case-analysis procedure for paired multimodal data.

\section{Limitations}

The interpretation of the proposed quantities depends on the fixed site, the nested global family, and the declared local family. Obstruction is relative to the site graph and the parameter-sheaf specification, while hardness is relative to the normalized nested hierarchy, so these choices must remain fixed across modality pairs. The controlled calibrations use identity restrictions and analytically tractable synthetic settings. On paired real data, encoder choice, sample coverage, graph construction, and any edge-dependent transports will also affect the resulting values.

\section{Conclusion}

We introduced two directed invariants for cross-modal compatibility on a fixed, modality-independent site. Projection hardness measures the global model complexity required to reach a target error across a nested hierarchy, and sheaf-Laplacian obstruction measures the parameter variation required by locally fitted maps within a declared local family. Under identity restrictions, obstruction reduces to vector-valued graph Dirichlet energy, with global sections providing the corresponding gluing interpretation. The same formulation accommodates edge-dependent transports between local parameter spaces.

We related obstruction to the graph spectral gap and excess global-map error, established non-transitivity through an explicit ReLU construction, and recovered the predicted behavior in controlled calibrations of hardness, obstruction scaling, and site sensitivity. Fixing the site, normalization, global hierarchy, and local sheaf specification makes these quantities comparable across directed modality pairs and allows mediation to be tested directly.


\newpage
\bibliography{references}

\end{document}